\documentclass{article}

\usepackage{PRIMEarxiv}

\usepackage[utf8]{inputenc} 
\usepackage[T1]{fontenc}    

\usepackage{graphicx}
\usepackage{subcaption}
\usepackage{algorithm}
\usepackage{color}
\usepackage{algpseudocode}
\usepackage{amsmath}
\usepackage{amsthm}
\usepackage{multirow}
\usepackage{bm}
\newtheorem{theorem}{Theorem}
\newtheorem{proposition}{Proposition}
\usepackage{amsfonts}
\usepackage{booktabs} 
\newtheorem{lemma}[theorem]{Lemma}
\usepackage{hyperref}       
\title{Flow Matching Posterior Sampling: A Training-free Conditional Generation for Flow Matching}
\author{
Kaiyu Song, \\
  Sun Yat-Sen University \\ \\
   \And
  Hanjiang Lai \\
  Sun Yat-Sen University \\\\
     \And
  Yan Pan \\
  Sun Yat-Sen University \\\\
     \And
  Kun Yue \\
  Yunnan University \\\\
     \And
  Jian Yin \\
  Sun Yat-Sen University \\\\
}

\begin{document}
\maketitle

\begin{abstract}
Training-free conditional generation based on flow matching aims to leverage pre-trained unconditional flow matching models to perform conditional generation without retraining. 
Recently, a successful training-free conditional generation approach incorporates conditions via posterior sampling, which relies on the availability of a score function in the unconditional diffusion model.
However, flow matching models do not possess an explicit score function, rendering such a strategy inapplicable. Approximate posterior sampling for flow matching has been explored, but it is limited to linear inverse problems. 
In this paper, we propose Flow Matching-based Posterior Sampling (FMPS) to expand its application scope. We introduce a correction term by steering the velocity field. This correction term can be reformulated to incorporate a surrogate score function, thereby bridging the gap between flow matching models and score-based posterior sampling. Hence, FMPS enables the posterior sampling to be adjusted within the flow matching framework. Further, we propose two practical implementations of the correction mechanism: one aimed at improving generation quality, and the other focused on computational efficiency.
Experimental results on diverse conditional generation tasks demonstrate that our method achieves superior generation quality compared to existing state-of-the-art approaches, validating the effectiveness and generality of FMPS.
\end{abstract}
\begin{figure*}[!ht]
    \centering
    \includegraphics[width=\linewidth]{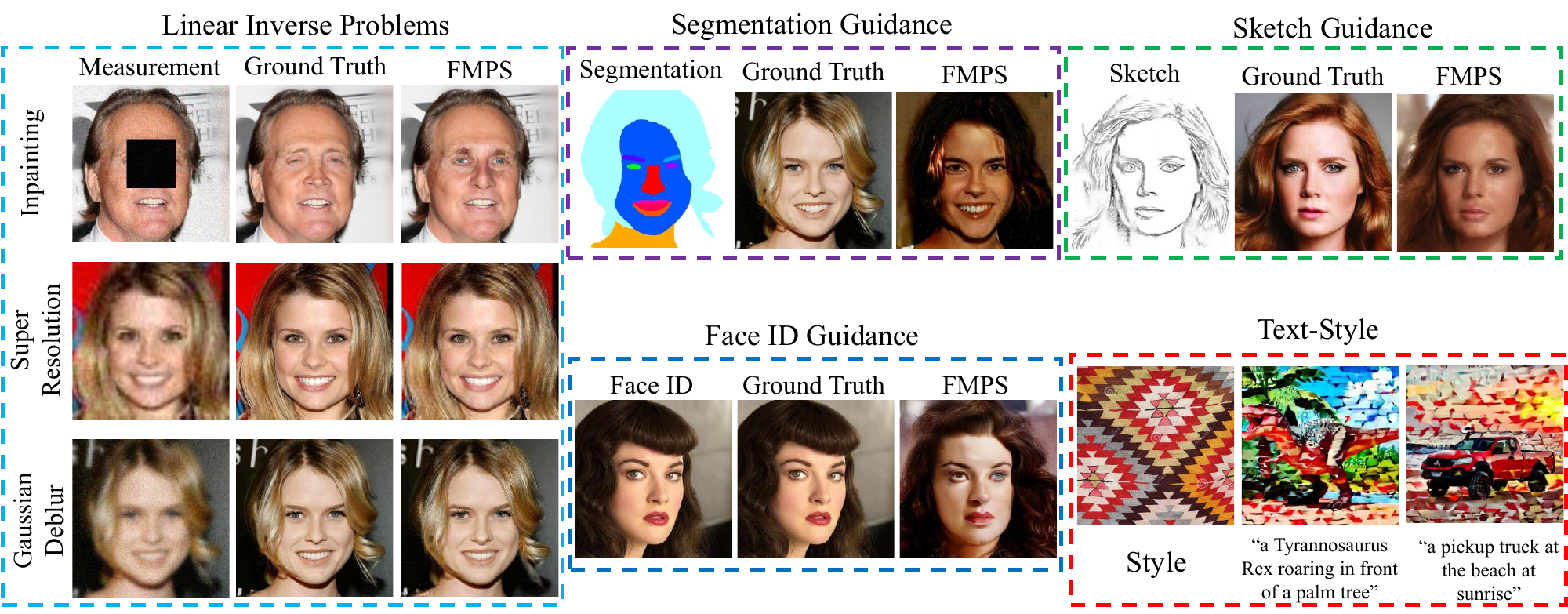}
    \caption{Our proposed FMPS in various conditional generation applications. FMPS has the same flexibility and generality as the training-free methods based on SDMs.}
    \label{fig:intro}
\end{figure*}

\section{Introduction}
\label{sec:intro}

Flow matching models (FMs)~\cite{rectified_flow} have achieved tremendous success in generative modeling. These models learn a velocity function~\cite{sd3.0} to move the Gaussian distribution to the target data distribution. Although FMs have achieved impressive results, these models are trained without conditions. 
When generating images or videos, users often prefer to create them according to specific requirements. Therefore, conditional generative diffusion models~\cite{sd,3d} have also received increasing attention recently. This is an essential task in generative modeling, which has a wide range of applications, e.g., linear/non-linear inverse problems~\cite{dps}, image editing~\cite {sdedit}, etc. 

A simple and effective approach to obtaining a conditional generative diffusion model is to leverage an already pre-trained unconditional diffusion model. Please note that Score function-based Diffusion Models (SDMs)~\cite{sde} and FMs represent two major categories of the unconditional diffusion models. In this paper, we only focus on conditional generation based on flow matching.

To leverage an unconditional diffusion model for conditional generation, a straightforward approach is to fine-tune this unconditional model by incorporating the condition into its backbone architecture~\cite{b-lora,dreambooth}. For example, D-flow~\cite{d-flow} guides the generation process by fine-tuning the start point of FMs. However, fine-tuning a diffusion model incurs significant computational costs~\cite{dps}. 

To reduce this cost, training-free conditional generation methods~\cite{zecon} have been proposed, which simply reuse the unconditional FMs to perform the conditional generation without any fine-tuning or retraining. For example, in unconditional SDMs~\cite{sde}, one successful training-free conditional generation approach is the posterior sampling~\cite{dps}. Posterior sampling leverages Bayes' rule to transform the unconditioned marginal distribution $p(x)$ to the unknown but desired conditional distribution $p(x|c)$, utilizing the score function inherent to SDMs. 
Further, Freedom~\cite{freedom} proposed using a simple energy function in the conditional term by computing the gradient of a distance metric between the current state and the condition. In this way, the reverse process reduces the distance between the current state and the condition, thereby achieving conditional generation. Later, MPGD~\cite{mpgd} has been proposed to reduce the computation cost of posterior sampling by leveraging manifold theory.

Posterior sampling enables a flexible mechanism for incorporating the conditions, as it only requires the definition of a differentiable distance metric between the current state and the condition. This flexibility allows training-free conditional generation based on SDMs to be applied to various conditional generative tasks, such as inversion tasks~\cite{add7} and style transfer tasks~\cite{RED}. 

However, FMs learn a velocity function and do not have an explicit score function. It hinders the training-free conditional FMs from using posterior sampling, as the bridge that connects the unconditional distribution $p(x)$ and the conditional distribution $p(x|c)$ is missing. Only a few works have attempted to rebuild this connection. For example, $\pi$-flow~\cite{pi_flow} proposed posterior sampling for FMs, but it is limited to linear inverse problems. Feng~\textit{et al.}~\cite{challenge} further introduced the mentor carol approximation to incorporate posterior sampling with the energy function, yet it also applies only to linear inverse problems.
Thus, existing posterior sampling methods for training-free conditional FMs remain limited in scope and cannot support flexible designs for diverse conditional generation tasks.

To address this problem, we propose a novel training-free conditional generation method for FMs, termed flow-matching posterior sampling (FMPS). FMPS enables posterior sampling to be as flexible as in SDMs, thereby extending conditional FMs beyond the domain of linear inverse problems. Our key insight is to introduce a correction term by steering the velocity field, which constructs a valid pathway linking the velocity function to the score function. This enables FMPS to achieve the same flexibility as SDMs in handling conditional generation.


Specifically, to incorporate the condition into the velocity field of unconditional FMs, we propose a correction term that provides an additional direction aligned with the condition, i.e., steering the velocity field to the condition. We then reformulate this correction term to resemble the score function in SDMs. 
Under this formulation, the entire framework reduces to posterior sampling in SDMs, where the conditional distribution can be formulated as the unconditional distribution and the energy function. 

Furthermore, we propose two variants of FMPS: \textit{gradient-aware} FMPS and \textit{gradient-free} FMPS. The gradient-aware FMPS improves generation quality by leveraging the gradient of the FMs, while the gradient-free FMPS is more computationally efficient by avoiding gradient computations. 
The experimental results on linear inverse problems, non-linear inverse problems, and text-to-image generation tasks demonstrate that FMPS achieves superior generation quality in various downstream tasks, verifying the validity and efficiency of FMPS. 

To sum up, the main contributions of the FMPS are:
\begin{itemize}
    \item We propose a novel training-free conditional generation method, called FMPS, for the flow diffusion models.
    \item We propose two straightforward approaches for implementing posterior sampling for FMs.
    \item The experimental results show that FMPS could solve board-range conditional generation tasks while maintaining high-quality generation.
\end{itemize}


\section{Related Work}
\label{sec:related work}

Diffusion-based conditional generation has become a key approach for generating desired samples using diffusion models~\cite{DDIM,sde,add1}, which has been widely explored in various applications. For instance, in inverse problems~\cite{ddrm, dps,add2,add3}, degraded images are used as conditions to restore or enhance image quality. In addition, text-to-image~\cite{sd, sd3.0, controlNet, controlNet++}, text-to-video~\cite{opensora}, and text-to-3D~\cite{3d} generation methods utilize text prompts as conditions to generate corresponding images, videos, and 3D objects, respectively. Depending on whether the model requires retraining for condition incorporation, existing conditional generation methods can be broadly categorized into training-based and training-free approaches. This paper focuses on the latter, training-free methods. 

Training-free methods~\cite{dps,freedom,mpgd,ddrm,pi_flow} aim to leverage pre-trained unconditional diffusion models without requiring additional training. Based on the type of pre-trained model, i.e., Score-based Diffusion Models (SDMs)~\cite{DDPM} and Flow Matching models (FMs)~\cite{add5,add4}, training-free conditional generation methods can be categorized into two groups: those based on SDMs and those based on FMs. In the following, we first provide an overview of SDMs and FMs in Subsection~\ref{twotypes}. We then discuss training-free conditional generation methods based on SDMs in Subsection~\ref{training_freeSDM} and those based on FMs in Subsection~\ref{training_freeFMs}.

\subsection{Two Types of Diffusion Models} 
\label{twotypes}

There are two broad categories of diffusion models: score function-based~\cite{sde,DDIM} and flow-matching-based~\cite{rectified_flow,sd3.0} diffusion models. The key distinction lies in the learning objective. Score-based diffusion models (SDMs) aim to learn a score function, which represents the divergence at the current state. This divergence, defined as the gradient of the log marginal distribution, guides the diffusion process toward the target data distribution~\cite{sde}. As a result, SDMs can be interpreted as sampling from a marginal distribution $p(x)$, a property that is crucial for training-free methods. In contrast, Flow matching models (FMs) learn a velocity field at each state~\cite{rectified_flow}, which characterizes the dynamics along a path connecting the target data and noise distributions. FMs utilize this velocity field to parameterize the transition path. Generally, a straighter path corresponds to faster convergence. Therefore, FMs can learn a near-straight path to enable both fast and high-quality sampling.

\subsection{Training-free Conditional Generation based on SDMs}
\label{training_freeSDM}

Most existing training-free conditional generation methods~\cite{dps,freedom,mpgd,freetuner,add6,add7} are based on the score function in SDMs. The score function enables models to sample from an unknown conditional distribution $p(x|c)$ via the known marginal distribution $p(x)$, i.e., the posterior sampling. The difference behind these methods is how to approximate $p(x|c)$ via $p(x)$. DPS~\cite{dps} proposed the Bayesian approximation method for $p(x|c)$ by leveraging the linear assumption in linear inversion. FreeDom~\cite{freedom} extended DPS by introducing the energy function. Such a design unleashed the training-free methods since we could approximate the $p(x|c)$ by defining a differentiable distance metric. In this case, SDM-based training-free method could solve the non-linear inverse problems~\cite{freedom} and more complex tasks such as style-transfer~\cite{MCG}. DSG~\cite{DSG} proposed the sphere theory to improve the posterior sample, and LGD-MC~\cite{lgd-mc} proposed a multi-sampling process to improve the posterior sample. However, posterior sampling required the gradient of the diffusion models, thus introducing a computational cost problem. MPGD~\cite{mpgd} introduced the linear manifold theory to avoid gradient calculation, thereby alleviating the computation cost. UFG~\cite{add8} propose a unified principle for SDMS-based methods to select the hyperparameter.

\subsection{Training-free Conditional Generation based on FMs} \label{training_freeFMs}

The FM paradigm has been introduced to improve upon SDMs by replacing the score function with a velocity field. However, this change disrupts the previously established link to the conditional distribution $p(x|c)$, which was naturally accessible through the score function in SDMs. As a result, FMs face challenges in reconstructing an effective connection between the marginal distribution $p(x)$ and the conditional distribution $p(x|c)$.

Several approaches have been proposed to address this issue. For instance, D-Flow~\cite{d-flow} introduced a lightweight framework that fine-tunes the initial point of the generation process using a given condition, effectively guiding the trajectory toward $p(x|c)$. In contrast, 
$\pi$Flow~\cite{pi_flow} first exploits the theoretical relationship between FMs and SDMs to re-establish a formulation of $p(x|c)$ based on $p(x)$, and then proposes a closed-form solution under the assumption of linear inversion. Feng \textit{et al.}~\cite{challenge} developed a novel theoretical framework that directly links $p(x|c)$ with the distribution $p(x)$. This approach leads to a closed-form expression for $p(x|c)$, which, however, is analytically intractable. To overcome this limitation, the authors further apply the mentor-color theory to approximate the intractable form, again relying on the linear inversion assumption.

Unlike existing approaches that often depend on the linear inversion assumption, we propose FMPS, a flexible framework that can be applied to diverse conditional generation tasks without such restrictions. 

\section{Preliminary}
\label{sec:pre}

\subsection{Training-free conditional SDMs}

\textbf{Unconditional SDMs.} In the forward process of diffusion models, SDMs define a Gaussian probabilistic path to add noise to the image as:
\begin{equation}
    x_{t} = \sqrt{\hat{\alpha}_{t}}x_{0} + \sqrt{(1 -\hat{\alpha}_{t})}\epsilon,
    \label{eq:diffusion_forward}
\end{equation}
where $\epsilon \sim \mathcal{N}(0,1)$, $\hat{\alpha}_{t}$ is the parameter of the noise scheduler related to the $t$-th time, $x_{0}$ is the image sampled from the data distribution.

In the reverse process, SDMs learn  to move the Gaussian data back to the target data via the score function as follows:
\begin{equation}
    dx_{t} = [f(t)x_{t}-\frac{1}{2}g(t)^2 \nabla_{x_{t}}\log p(x_{t})]dt,
    \label{eq:diffusion_reverse}
\end{equation}
where $f(t)$ and $g(t)$ are the hyper-parameters in the noise scheduler related to the $\hat{\alpha}_{t}$. $\nabla_{x_{t}}\log p(x_{t})$ is the score function modeled by the neural network. Eq.~\ref{eq:diffusion_reverse} is a divergence field that guides Gaussian noise flow to the data distribution.

\textbf{Training-free conditional SDMs.} Given a condition $c$ such as a text prompt, the conditional score function that represents the divergence of the condition distribution can be defined as:
\begin{equation}
    \nabla_{x_{t}}\log p(x_{t}|c).
    \label{eq:conditional_diffusion}
\end{equation}
Then, by the Bayesian rule, Eq.~\ref{eq:conditional_diffusion} could be split as:
\begin{equation}
    \nabla_{x_{t}}\log p(x_{t}|c) = \nabla_{x_{t}}\log p(x_{t}) + \nabla_{x_{t}}\log p(c|x_{t}).
    \label{eq:bayesian}
\end{equation}
$\nabla_{x_{t}}\log p(x_{t})$ is exactly the score function. $\nabla_{x_{t}}\log p(c|x_{t})$ is the conditional term, which is unknown and needs to be estimated.

\textbf{Posterior Sampling on SDMs.} Previous works~\cite{freedom,dps} proposed posterior sampling for the estimation of $\nabla_{x_{t}}\log p(c|x_{t})$. In this paper, we follow Freedom~\cite{freedom}, which demonstrates that we could use a differentiate distance metric $D(\ast,\ast)$ to measure the distance between $x_{t}$ and $c$. That is, $\nabla_{x_{t}}\log p(c|x_{t})$ could be directly estimated by $D(\ast,\ast)$ as:
\begin{equation}
    \nabla_{x_{t}}\log p(c|x_{t}) \approx r\nabla_{x_{t}}D(\hat{x}_{0|t},c),
    \label{eq:freedom}
\end{equation}
where $r$ is the hyperparameter to control the strength of the given condition $c$. $\hat{x}_{0|t}$ is the data distribution prediction for $x_{t}$, which is defined as:
\begin{equation}
     \hat{x}_{0|t} = \frac{1}{\sqrt{\hat{a}_{t}}}(x_{t} - (1 - \hat{a}_{t}) \epsilon_{\theta}(x_{t},t)).
     \label{eq:mmse}
\end{equation}

Therefore, by Eq.~\ref{eq:bayesian} - Eq.~\ref{eq:mmse}, we achieve the training-free conditional generation based on SDMs.

\subsection{Flow Matching Models} In the forward process, FMs move the data distribution $p_{0}$ to the Gaussian noise distribution $p_{1}$ as follows:
\begin{equation}
    x_{t} = a_{t}x_{0} + b_{t}\epsilon,
    \label{eq:flow_forward}
\end{equation}
where $\epsilon \sim p_{1}= \mathcal{N}(0,1)$, $x_{0} \sim p_{0}$ and $x_{t}$ is the intermediate result. $a_{t}$ and $b_{t}$ are schedule parameters~\cite{rectified_flow}.

In the reverse process, the velocity function is first defined as:
\begin{equation}
    dx_{t} = v_{\theta}(x_{t},t)dt,
    \label{eq:flow_reverse}
\end{equation}
where $v_{\theta}(x_{t},t)$ is the velocity function, which is modeled by the neural networks. Then, the reverse process could be calculated by the ODE solver, such as the  Euler solver:
\begin{equation}
    x_{t-1} = x_{t} + h \times v_{\theta}(x_{t},t),
    \label{eq:euler_flow}
\end{equation}
where $h$ is the time interval between the $t$-th and $(t-1)$-th time steps.

\subsection{An Inherent Challenge for Conditional Guidance for FMs}

There is no score function in FMs as shown in Eq.~\ref{eq:flow_reverse}. In this condition, the posterior sampling is invalid because there is no direct way to formulate $\nabla_{x}\log p(c|x)$ term as in Eq.~\ref{eq:bayesian}. Thus, we cannot ensure the reverse process is guided to the direction of the condition. 

To tackle this, Feng~\textit{et al.}~\cite{challenge} proposed a framework to introduce the guidance term as follows:
\begin{equation}
    dx_{t} = v_{\theta}(x_{t},t) + g_{t}(x_{t})dt,
    \label{eq:general formulation}
\end{equation}
where $v_{\theta}(x_{t},t)$ is the velocity function and the $g_{t}$ is the guidance term to model the direction for the condition. 
They showed that the energy function $J$ can be used to approximate $g_{t}(x_{t})$ via $p(x|c) \sim \frac{1}{Z}p(x)e^{-J(x,c)}$, where $Z$ is the normalization constant. The definition is:
\begin{equation}
    g_{t}(x_{t})=\int v_{t|z}(x_{t}|z)(\frac{\frac{1}{Z}p(z|x_{t})e^{-J(x,c)}}{p(z|x_{t})}-1)p(z|x_{t})dz.
    \label{eq:challenge}
\end{equation}
where $z =(x_{0},\epsilon)$. 
Eq.~\ref{eq:challenge} is intractable due to the need to calculate the guidance term to travel all possible data pairs $z$~\cite{challenge}. Thus, $g_{t}(x_{t})$ is limited to the linear inverse problem. 

In this paper, unlike Feng~\textit{et al.}~\cite{challenge} that used Monte Carlo sampling to approximate posterior sampling, we show that our method is a direct and straightforward closed-form approximation via the connection between SDMs and FMs.

\section{Methodology}
\label{sec:method}
In this section, we introduce a simple and efficient method to calculate our correction term $\Delta(x_{t},c)$ via the posterior sampling. The key insight is to rebuild the connection between the reverse process of FMs and the score function. Leveraging this, FMPS contains two key steps: 
1) We reformulate the velocity function in FMs as the score function in SDMs. This ensures the reverse process of FMS can be defined as the distribution $p(x)$ in SDMs. Thus, we can formulate our condition-related correction term $\Delta(x_{t},c)$.
2) We propose the FMPS that achieves the posterior sampling for FMs. This ensures FMPS can also sample the conditional distribution $p(x|c)$ via the correction term, thereby correctly achieving conditional generation. The existing posterior sampling theoretically proves the validity of FMPS.

\subsection{The Condition-related Correction Term} 

Eq.~\ref{eq:flow_reverse} represents the velocity field connecting the data distribution and the Gaussian noise distribution, where $v_{\theta}(x_t, t)$ is the direction to guide the movement of the $x_{t}$. 
To introduce the condition $c$ into the reverse process,  we propose to add a new condition-related correction term into the velocity function as:
    \begin{equation}
        dx_{t} = [v_{\theta}(x_{t},t) + r \Delta(x_{t},c)]dt,
        \label{eq:start}
    \end{equation} 
where $\Delta(x_{t},c)$ is the correction term to offer the additional direction to $c$. This way explicitly introduces the condition into the FMs. And $r$ is the scale to control the strength. The challenge is how to ensure $\Delta(x_{t},c)$ could offer a valid direction for the condition.


\textbf{Score function ensures valid guidance direction.} Similar to the score function, the valid guidance direction is that the correction term could also sample from $p(x)$ to $p(x|c)$ as in SDMs. To achieve it, we need to show that the velocity function in FMs can be reformulated as the score function in SDMs. 


Firstly, we have the following proposition to introduce: 
\begin{proposition}[\cite{sit}]
    Given any ODE in FMs, there exists a valid probabilistic diffusion path, which could be represented related to the score function as follows:
    \begin{equation}
        v_{\theta}(x_{t},t) = \frac{\dot{b}_{t}}{b_{t}}x_{t} - a_{t}(\dot{a}_{t} - \frac{\dot{b}_{t}}{b_{t}}a_{t})\nabla_{x_{t}}\log p(x_{t}),
        \label{eq:main proposition}
    \end{equation}
where $\dot{b}_{t}$ is $\frac{\partial b_{t}}{\partial t}$ and $\dot{a}_{t}$ is $\frac{\partial a_{t}}{\partial t}$. 
\label{pro:main proposition}
\end{proposition}
Proposition~\ref{pro:main proposition} explicitly defines the score function in the reverse process of the FMs. Thus, we could naturally connect the velocity function of FMs to the score function. 

The condition now could be introduced into the velocity function via Bayes' rule. Thus, we can define our correction term in the following Lemma.
\begin{lemma}
    By introducing $c$, the ODE of FMs could be re-formulated based on the score function as follows:
    \begin{equation}
        dx_{t} = [\lambda x_{t} - \beta_{t}(\nabla_{x_{t}}\log p(x_{t})+\nabla_{x_{t}}\log p(c|x_{t})) ]dt,
        \label{eq:new_ODE}
    \end{equation}
    where $\lambda = \frac{\dot{b}_{t}}{b_{t}}$ and $\beta = a_{t}(\dot{a}_{t} - \frac{\dot{b}_{t}}{b_{t}}a_{t})$
    \label{lemma:new_ODE}
\end{lemma}
\begin{proof}
    By Eq.~\ref{eq:flow_reverse} and Eq.~\ref{eq:main proposition}, we could get the unconditional FMs formulation via the score function as:
    \begin{equation}
        dx_{t} = [\frac{\dot{b}_{t}}{b_{t}}x_{t} - a_{t}(\dot{a}_{t} - \frac{\dot{b}_{t}}{b_{t}}a_{t})\nabla_{x_{t}}\log p(x_{t})]dt.
        \label{eq:apply_proposition}
    \end{equation}
    To simplify Eq.~\ref{eq:apply_proposition}, we set $\lambda = \frac{\dot{b}_{t}}{b_{t}}$ and $\beta_{t} = a_{t}(\dot{a}_{t} - \frac{\dot{b}_{t}}{b_{t}}a_{t})$, thus having:
    \begin{equation}
        dx_{t} = [\lambda x_{t} - \beta_{t}\nabla_{x_{t}}\log p(x_{t}) ]dt.
        \label{eq:unconditional_fms}
    \end{equation}
    We introduce the $c$ by changing the marginal distribution $p(x_{t})$ to the conditional distribution $p(x_{t}|c)$ following the posterior sampling as follows:
    \begin{equation}
        dx_{t} = [\lambda x_{t} - \beta_{t}\nabla_{x_{t}}\log p(x_{t}|c)]dt.
        \label{eq:conditional_fms}
    \end{equation}
    Then, by Bayes' rule, we have:
    \begin{equation}
        p(x_{t}|c) = \frac{p(c|x_{t})p(x_{t})}{\int{p(c)}}.
        \label{eq:bayes_flow}
    \end{equation}
    We extent Eq.~\ref{eq:bayes_flow} as:
    \begin{equation}
    \begin{split}
                \nabla_{x_{t}}\log p(x_{t}|c) &= \nabla_{x_{t}}\log \frac{p(c|x_{t})p(x_{t})}{\int{p(c)}}\\
                &=\nabla_{x_{t}}[\log p(c|x_{t}) + \log p(x_{t}) - \log\int{p(c)}].
    \end{split}
    \label{eq:finnal_bayes}
    \end{equation}
    Since $\int{p(c)}$ is irrelevant to the $x_{t}$, $\nabla_{x_{t}}\int{p(c)}$ is zero. With Eq.~\ref{eq:finnal_bayes} to Eq.~\ref{eq:conditional_fms}, we have:
    \begin{equation}
        dx_{t} = [\lambda x_{t} - \beta_{t}(\nabla_{x_{t}}\log p(x_{t})+\nabla_{x_{t}}\log p(c|x_{t})) ]dt.
    \end{equation}
    Thus, we finish the proof.
\end{proof}
Lemma~\ref{lemma:new_ODE} redefines a new ODE based on the score function for FMs. Now, we re-organize Eq~\ref{eq:new_ODE} as follows:   
\begin{equation}
\begin{split}
        dx_{t} &= [(\lambda x_{t} - \beta_{t}\nabla_{x_{t}}\log p(x_{t})) -\beta_{t}\nabla_{x_{t}}\log p(c|x_{t}) ]dt\\
        &=[v_{\theta}(x_{t},t) -\beta_{t}\nabla_{x_{t}}\log p(c|x_{t})]dt.
        \label{eq:correction_term}
\end{split}
\end{equation}
It could be found that Eq~\ref{eq:correction_term} and Eq.~\ref{eq:start} could be matched by $\Delta(x_{t},c) = -\beta_{t}\nabla_{x_{t}}\log p(c|x_{t})$. Then, $\beta_{t}\nabla_{x_{t}}\log p(c|x_{t})$ exactly represents the direction from $x_{t}$ to $c$. Thus, we could use $-\beta_{t}\nabla_{x_{t}}\log p(c|x_{t})$ to formulate the correction term by:
\begin{equation}
    dx_{t} = [v_{\theta}(x_{t},t) - r \beta_{t}\nabla_{x_{t}}\log p(c|x_{t})]dt.
    \label{eq:theory_support}
\end{equation}
Eq.~\ref{eq:theory_support} proves that the reverse process is always sampled from $p(x|c)$~\cite{freedom} by $\Delta(x_{t},c) = - r \beta_{t}\nabla_{x_{t}}\log p(c|x_{t})$, thereby ensuring the validity of $\Delta(x_{t},c) $. 

Compared to Eq.~\ref{eq:challenge}, our method is more versatile and can be applied to a wide range of generation tasks, rather than being restricted to linear inversion problems. Moreover, our approach does not require access to all unknown data pairs, making it computationally tractable.

\subsection{FMPS}

\textbf{Posterior sampling for training-free conditional FMs}. The $\nabla_{x_{t}}\log p(c|x_{t})$ naturally allows the correction to sample from $p(x|c)$, thereby finishing the conditional generation.
Our final step is to calculate $\nabla_{x_{t}}\log p(c|x_{t})$, and calculating this term could follow the similar steps of the posterior sampling~\cite{challenge} in SDMs again. Therefore, our framework could be regarded as FMs-based posterior sampling. 

Now, we could also approximate $\nabla_{x_{t}}\log p(c|x_{t})$ as \cite{freedom} by:
\begin{equation}
    \nabla_{x_{t}}\log p(c|x_{t}) \approx \nabla_{x_{t}}D(x_{t},c),
\end{equation}
where $D(\ast,\ast)$ is the distance metric. Similar to the Freedom~\cite{freedom}, posterior sampling enables the use of the energy function (e.g., gradient of the distance metric) to approximate the gradient of the conditional distribution $p(c|x_{t})$. Thus, FMPS enables posterior sampling as flexible as the DPMs-based posterior sampling.

Finally, $x_{t}$ may retain some Gaussian noise to add the bias for the gradient of the distant metric and offer the wrong direction. One way to remove the Gaussian noises is the $x_{0}$ prediction~\cite{freedom} by changing $x_{t}$ to $\hat{x}_{0|t}$ as follows:
\begin{equation}
    D(x_{t},c) \approx D(\hat{x}_{0|t},c).
\end{equation}

FMPS offers \textit{gradient-aware} FMPS (FMPS-gradient) and \textit{gradient-free} FMPS (FMPS-free) versions to calculate $\hat{x}_{0|t}$ by whether $ \nabla_{x_{t}}\log p(c|\hat{x}_{0|t})$ needs to introduce the gradient of $v_{\theta}(x_{t},t)$. 

The FMPS-gradient derives from the one-step Euler solver, which could be represented as follows:
\begin{equation}
    \hat{x}_{0|t} = x_{t} - tv_{\theta}(x_{t},t).
    \label{eq:gradient_aware_solver}
\end{equation}
Eq.~\ref{eq:gradient_aware_solver} directly solves the $t \rightarrow 0$ interval to finish the $x_{t}$ estimation. The advantage of gradient-aware estimation is introducing the gradient information of $\frac{\partial v(x_{t},t)}{\partial x_{t}}$ when calculating $\nabla_{x_{t}}D(\hat{x}_{0|t},c)$. Gradient information will benefit the conditional tasks~\cite{gradient_pgd}. Thus, it can be regarded as a test-time scaling method~\cite{ttaf} since calculating the gradient information introduces extra computational costs.

This motivates us to propose our FMPS-free. We redefine the forward process of FMs as follows:
\begin{equation}
    x_{0} = \frac{1}{a_{t}}(x_{t} - b_{t}x_{1}).
    \label{eq:forward_process_flow_x_0}
\end{equation}
Eq.~\ref{eq:forward_process_flow_x_0} shows $x_{0}$ is related to the intermediate state $x_{t}$ and the initial state $x_{1}$. Since $x_{1}$ is known and unchanged during the overall generation process and $x_{t}$ is already calculated in the $t$-th time step, we could directly use Eq.~\ref{eq:forward_process_flow_x_0} to estimate $\hat{x}_{0|t}$ as:
\begin{equation}
    \hat{x}_{0|t} = \frac{1}{a_{t}}(x_{t} - b_{t}x_{1}).
    \label{eq:gradient_free_solver}
\end{equation}
Eq.~\ref{eq:gradient_free_solver} decouples the connect between $\hat{x}_{0|t}$ and $v_{\theta}(x_{t},t)$ compared to Eq.~\ref{eq:gradient_aware_solver}. Thus, it avoids calculate the gradient $\frac{\partial v(x_{t},t)}{\partial x_{t}}$ when calculating $\nabla_{x_{t}}D(\hat{x}_{0|t},c)$, which achieves the gradient-free estimation. This enables saving a lot of the computation costs but needs to skip the first step since when $t=1$, $a_{t}=0$. 

\textbf{Normalization trick for FMPS.} FM is sensitive to the scale of the direction. FMPS has two directions from $v_{\theta}(x_{t},t)$ and the correction term. We propose a normalization trick to alleviate the scale problem. Concretely, ODE decides the strength of $v(\ast,\ast)$ could be measured by its norm since $v(\ast,\ast)$ represents a forward direction in ODE. Thus, we could measure the strength of $\nabla_{x_{t}}D(\hat{x}_{0|t},c)$.  We have:
\begin{equation}
    g^1 = \frac{||v(x_{t},t)||_{2}}{||\nabla_{x_{t}}D(\hat{x}_{0|t},c)||_{2}}\nabla_{x_{t}}D(\hat{x}_{0|t},c).
    \label{eq:tricks}
\end{equation}
Eq.~\ref{eq:tricks} limits the strength of $\nabla_{x_{t}}D(\hat{x}_{0|t},c)$ under the bound of $||v(x_{t},t)||_{2}$, where $\nabla_{x_{t}}D(\hat{x}_{0|t},c)$ will not too strong and too weaken.
\begin{algorithm}[t]
    \caption{The algorithm for FMPS-gradient} \label{al:gradient-aware}
    \begin{algorithmic}[1]
     \Statex \textbf{Input:} $v_{\theta}(\ast,\ast)$, $T$, $c$, $D(\ast,\ast)$, $r$
     \Statex \textbf{Output:} $x_{0}$
    \State Initialize time interval $n \gets \frac{1}{T}$
    \State $x_{T}\sim\mathcal{N}(0,1)$
    \State $x_{t} \gets x_{T}$
    \For{$t$, $t-1$ in pair $[(1,1-n),(1-n, 1-2n),(1-in, 1-(i+1)n)...,(1-(T-1)n,0)]$} 
        \State Calculate $a_{t}$, $b_{t}$, $\dot{a}_{t}$ and $\dot{b}_{t}$ by Eq.~\ref{eq:flow_forward}.
        \State Calculate $\lambda_{t}$ and $\beta_{t}$ by Eq.~\ref{eq:new_ODE}.
        \State Calculate $v_{\theta}(x_{t},t)$.
        \State Calculate $\hat{x}_{0|t}$ by Eq.~\ref{eq:gradient_aware_solver}.
        \State Calculate $g^{1}$ by Eq.~\ref{eq:tricks}.
        \State Calculate $x_{t-1}$ using Euler solver by Eq.~\ref{eq:ppbf_euler}.
        \State $x_{t} \gets x_{t-1}$.
    \EndFor 
    \Statex\textbf{Return:} $x_{0}$
    \end{algorithmic}
\end{algorithm}
Thus, we could get the conditional generation ODE for FMPS as follows:
\begin{equation}
     dx_{t} = [v_{\theta}(x_{t},t) -r\beta_{t}g^{1}]dt. 
     \label{eq:ppbf_ode}
\end{equation}
The related Euler solve for Eq.~\ref{eq:ppbf_ode} is:
\begin{equation}
    x_{t-1} = x_{t} + h(v_{\theta}(x_{t},t) - r\beta_{t}g^{1}).
    \label{eq:ppbf_euler}
\end{equation}

The overall framework for FMPS based on gradient-aware and gradient-free estimations for $\hat{x}_{0|t}$ is summarized as Algorithm 1 and Algorithm 2, respectively. To sum up, FMPS first proposes the correction term to introduce $c$. Then, FMPS leverages the score function to formulate the correction term, where approximating the correction term returns to the posterior sampling framework again. Therefore, FMPS proposes the two methods to achieve the posterior sampling for FMs.

\begin{algorithm}[t]
    \caption{The algorithm for FMPS-free} \label{al:gradient-free}
    \begin{algorithmic}[1]
     \Statex \textbf{Input:} $v_{\theta}(\ast,\ast)$, $T$, $c$, $D(\ast,\ast)$, $r$
     \Statex \textbf{Output:} $x_{0}$
    \State Initialize $n$, $x_{T}$, and $x_{t}$. 
    \State  $ \text{Isfirst} \gets \text{True}$ .
    \For{$t$, $t-1$ in pair $[(1,1-n),(1-in, 1-(i+1)n),...,(1-(T-1)n,0)]$} 
        \State Calculate $a_{t}$, $b_{t}$, $\dot{a}_{t}$, $\dot{b}_{t}$, $\lambda_{t}$, $v_{\theta}(x_{t},t)$ and $\beta_{t}$.
        \If{Isfirst}
            \State Calculate $x_{t-1}$ by Eq.~\ref{eq:euler_flow} \Comment{Skip first step.}
            \State$ \text{Isfirst} \gets \text{False}$.
            \Else
            \State Calculate $\hat{x}_{0|t}$ by Eq.~\ref{eq:gradient_free_solver}.
             \State Calculate $x_{t-1}$  Eq.~\ref{eq:ppbf_euler}.
        \EndIf
        \State $x_{t} \gets x_{t-1}$.
    \EndFor 
    \Statex\textbf{Return:} $x_{0}$
    \end{algorithmic}
\end{algorithm}

\begin{table*}
    \caption{Qualitative evaluation of linear inverse problems on CelebA-HQ, where SR (4X) represents the super-resolution (4X), and Deblur represents Gaussian Deblur.}
    \label{tab:linear}
    \centering
    \begin{tabular}{c c c c c c c c c c c}
    \toprule
        \multirow{2}{4em}{\textbf{Method}}  &  \multicolumn{3}{c}{\textbf{Inpainting}} & \multicolumn{3}{c}{\textbf{SR (4x)}} & \multicolumn{3}{c}{\textbf{Deblur}}  \\
     &\textbf{SSIM}$\uparrow$&\textbf{LPIPS}$\downarrow$&\textbf{FID}$\downarrow$ &\textbf{SSIM}$\uparrow$&\textbf{LPIPS}$\downarrow$&\textbf{FID}$\downarrow$ &\textbf{SSIM}$\uparrow$&\textbf{LPIPS}$\downarrow$&\textbf{FID}$\downarrow$ \\
    \midrule
         Score-SDE  &0.33& 0.63& 94.33 &0.58&0.39&53.22&0.63&0.36&66.81 \\
         ILVR & - &-& - & 0.74 &0.28&52.82&-&-&-   \\
         DPS & 0.610 &0.38&58.89 &  0.50 &0.46 &56.08& 0.58 &0.38& 52.64 \\
         LGD-MC & - &0.16& 28.21 & - &0.23&34.44&-&0.23& 32.57\\
         MPGD & 0.75 & 0.22& 11.83&  \textbf{0.79}& 0.25& 60.21&  0.77 &0.23& 51.59 \\
         DPS+DSG& - &0.12&  15.77 & - &0.21&30.30&-&0.21& 28.22\\
         $\Pi$Flow &0.87 &0.17& 12.56 & 0.79 &0.23&19.83&0.80&0.22& \textbf{12.6} \\
         FMPS-free (Our)&0.92 &0.03& 8.20 & 0.77 &0.14&24.39&0.76&0.13&20.11\\
        FMPS-gradient (Our)&\textbf{0.95} &\textbf{0.02}& \textbf{3.29} & 0.78&\textbf{0.11}&\textbf{17.83}&\textbf{0.77}&\textbf{0.11}&18.49\\
    \bottomrule
    \end{tabular}
\end{table*}

\begin{table*}
    \caption{Qualitative evaluation of non-linear inverse problems on CelebA-HQ, where Segmentation is the segmentation maps guided generation, Sketch is the sketch image guided generation, and Face ID is the classifier logit for face image guided generation.}
    \label{tab:non-linear}
    \centering
    \begin{tabular}{c c c c c c c c}
    \toprule
        \multirow{2}{4em}{\textbf{Method}}  &  \multicolumn{2}{c}{\textbf{Segmentation}} & \multicolumn{2}{c}{\textbf{Sketch}} & \multicolumn{2}{c}{\textbf{Face ID}}  \\
     &\textbf{Distance}$\downarrow$&\textbf{FID}$\downarrow$ &\textbf{Distance}$\downarrow$&\textbf{FID}$\downarrow$ &\textbf{Distance}$\downarrow$&\textbf{KID}$\downarrow$ \\
    \midrule
         FreeDom  & 1657.0& 38.65 &34.12& 52.18&0.57&0.0452 \\
         LGD-MC &2088.5&38.99 &49.46&54.47& 0.68 & 0.0445\\
         MPGD  & 1976.0&39.81&  37.23 &54.18  & 0.58 &0.0445  \\
         FMPS-free (Our)  & 1427.5& 40.36&  \textbf{19.25} & 60.61  & 0.46 &0.0349  \\
         FMPS-gradient (Our)&\textbf{1213.3} &\textbf{34.10}&19.31&  \textbf{52.01}&\textbf{ 0.39}&\textbf{ 0.0287}\\
    \bottomrule
    \end{tabular}
\end{table*}

\section{Experimental Results}
\label{sec:experiment}
\subsection{Implementation Details}
\textbf{Linear inverse problems.} Following the previous methods~\cite{freedom}, we evaluate FMPS on the CelebA-HQ dataset~\cite{celehq} to test its performance on the linear inverse problems. The linear inverse problems include \textbf{Inpainting (BOX)}, \textbf{Super-resolution (4X)}, and \textbf{Gaussian Deblur (kernel size: 61, intensify: 3)}. The baselines are from the latest methods, including Score-SDE~\cite{sde}, ILVR~\cite{sde}, LGD-MC~\cite{lgd-mc}, DPS~\cite{dps}, MGPD~\cite{mpgd}, and DPS+DSG~\cite{DSG}. We evaluate FMPS-free and FMPS-gradient on the linear inverse problem.

\begin{table*}
    \caption{Ablation study for comparing the FMPS-gradient and FMPS-free based on the SR (4x) on AFHQ, where MEM is the memory cost. FMPS-free (w/o) does not use the normalization trick.}
    \label{tab:aware_free}
    \centering
    \begin{tabular}{c c c c c c}
    \toprule
        \textbf{Method}  &  \textbf{SSIM} $\uparrow$ & \textbf{LPIPS}$\downarrow$ & \textbf{FID}$\downarrow$ & \textbf{Times (s)}$\downarrow$ & \textbf{MEM} (GB)$\downarrow$ \\
    \midrule
         FMPS-gradient  & \textbf{0.74}& \textbf{0.12} &  \textbf{7.68}& 15 &5.1\\
         FMPS-free & 0.69 & 0.13& 7.82&\textbf{7}& \textbf{3.0}\\
         FMPS-free (w/o) & 0.55& 0.15& 41.71&\textbf{7}& \textbf{3.0}\\
    \bottomrule
    \end{tabular}
    \label{fig:ala_trade}
\end{table*}
\textbf{Non-linear inverse problem.} Following the previous methods, we also evaluate FMPS on the CelebA-HQ dataset to test its performance on the non-linear inverse problems. The non-linear inverse problems include \textbf{Segmentation maps}, \textbf{Sketch}, and \textbf{Face ID}. We use the same pre-trained FMs as the linear inverse problems. The baselines are from the latest methods, including Freedom~\cite{freedom}, LGD-MC~\cite{lgd-mc}, and MPGD~\cite{mpgd}. We use FMPS-free and FMPS-gradient for the non-linear inverse problem.

\textbf{Text-style Generation.} Text-style generation guides the generation using text prompts and the given style image. We first evaluate FMPS on the 1,000 text-image pairs. The text prompts are randomly selected from PartiPrompts dataset~\cite{freedom}. Then, the style images are from FreeDom, following the previous works. We use FMPS-free for this task to show the potential of FMPS in large FMs since FMPS-gradient cannot afford the GPU cost to calculate the gradient of Stable Diffusion 3.0.

\textbf{Evaluation Metrics.}  
For linear inverse problems, we use structural similarity (SSIM)~\cite{SSIM}, LPIPS~\cite{lpips}, and Fréchet inception distance (FID)~\cite{fid} as the evaluation metrics to measure the generation quality. For non-linear inverse problems, we use the norm distance between the condition and generated images, FID, and kernel inception distance (KID)~\cite{kid} as the metrics. Specifically, the distance norm in Face ID tasks is the classifier logit. For the text-style generation, we use the CLIP score~\cite{CLIP} and the Gram matrix~\cite{freedom} to measure the distance between images and text prompts and style images, respectively.

\textbf{Distance metric.} The distance metric used in this paper is the same as the previous works~\cite{freedom}. The details are in the supplementary.

\begin{table}
    \caption{Qualitative evaluation of style text-to-image generation, where CLIP~\cite{CLIP} is the clip score between the text prompt and the image. Style is the Gram matrix~\cite{freedom} value between generated and style images.}
    \label{tab:style}
    \centering
    \begin{tabular}{c c c c}
    \toprule
        \textbf{Method}  &  \textbf{Style} $\downarrow$ & \textbf{CLIP}$\uparrow$  \\
    \midrule
         FreeDom  & 498.8& 30.14 \\
         LGD-MC &\textbf{404.0}& 21.16\\
         MPGD  & 441.0& 26.61 \\
         FMPS-free (Our)&435.2 &\textbf{32.65}\\
    \bottomrule
    \end{tabular}
\end{table}

\textbf{Pre-trained Models and settings.} To ensure reproduction, we use the pre-trained checkpoint trained in the CelebA-HQ dataset from official Rectified flow~\cite{rectified_flow} for the linear and non-linear inverse problems. For the style task, we use Stable Diffusion 3 to show the potential of FMPS. The time steps $T$ for all linear and non-linear inverse problem methods, including FMPS, are set to $T=100$ without specific illustration. All the experiments were run on a single RTX 4090 GPU.

\subsection{Qualitative Results}

\textbf{Linear inverse problems}. To show the ability of the FMPS in the linear inverse problem, we report the comprehensive qualitative evaluation in Table~\ref{tab:linear}. It can be noticed that both FMPS-free and FMPS-gradient generate high-quality images by leveraging the potential of FMs. Concretely, FMPS-free achieves the highest performance on all metrics in Inpainting tasks. Then, in super-resolution tasks, FMPS-free reduces the LPIPS from 0.21 to 0.14 while improving the FID from 30.3 to 24.39. In the end, in the Gaussian Deblur tasks, FMPS-free reduces the LPIPS from 0.21 to 0.13 while improving the FID from 28.22 to 20.11. Compared to the FMPS-free, the FMPS-gradient makes a further improvement since the FMPS-gradient leverages the gradient of the FMs. Concretely, FMPS-gradient improves the 4.91 FID metric in Inpainting tasks. Then, in the super-resolution and Gaussian Deblur tasks, the FMPS-gradient achieves the best performance. These results prove the validity of the FMPS in the linear inverse problems.

\textbf{Non-linear inverse problems}. We further explore the potential of the FMPS in non-linear inverse problems. We report the comprehensive qualitative evaluation in Table~\ref{tab:non-linear}. It can be noticed that FMPS-free and FMPS-gradient remain in the high-level generation. Concretely, FMPS-gradient achieves state-of-the-art (SOTA) performance in the segmentation task. Meanwhile, in the sketch task, FMPS-gradient improves by an average 43\% in the Distance metric while keeping the lowest FID. Meanwhile, FMPS-free greatly improves the distance metric, although the FID has increased. This proves that the gradient information is useful.

\textbf{Text-style generation}. To prove that the FMPS could work on the real generation tasks, we first report the comprehensive qualitative evaluation in Table~\ref{tab:style} using 1000 text-image pairs under nine style types. It can be observed that FMPS-free more closely adheres to the text prompt guidelines, achieving the highest CLIP score while maintaining a low style distance. Then, to show the ability of FMPS-free, we report the qualitative results using the complex text prompts shown in Fig.~\ref{fig:all_style}. It can be noticed that FMPS-free achieves a better generation since it maintains a similar style, following the style of images while satisfying the text prompts.

\textbf{Further comparison between FMs and SDMs.} To ensure the improvement from FMPS instead of the difference model ability between FMs and SDMs, we make a comparison shown in Table~\ref{tab:fair_comparison}. The setting ensure that both methods for FMs and SDMs use the same.

\begin{table}
    \caption{Qualitative evaluation of style text-to-image generation. We directly leverage the DPM-based solver, i.e., DPM-solver++ 2M, for FM to make a fair comparison. This solver will first convert the FM to a score-based diffusion paradigm and then implement the DPM-solver. We then implement MPGD in SD3 (MPGD-SD3). Such a setting ensures the same backbone during the training-free condition for both FM and DPMs. The results show that FMPS beats the MPGD and makes a better trade-off. This result also shows the gap that directly implementing the DPMs method on FMs models will degrade the performance of FMs, which enhances the contribution of FMPS. }
    \label{tab:fair_comparison}
    \centering
    \begin{tabular}{c c c c}
    \toprule
        \textbf{Method}  &  \textbf{Style} $\downarrow$ & \textbf{CLIP}$\uparrow$  \\
    \midrule
         MPGD-SD3  & 439.3 & 24.34 \\
         FMPS-free (Our)&\textbf{435.2} &\textbf{32.65}\\
    \bottomrule
    \end{tabular}
\end{table}

\subsection{Quantitative Results}

\begin{figure*}
    \centering
    \includegraphics[width=0.8\linewidth]{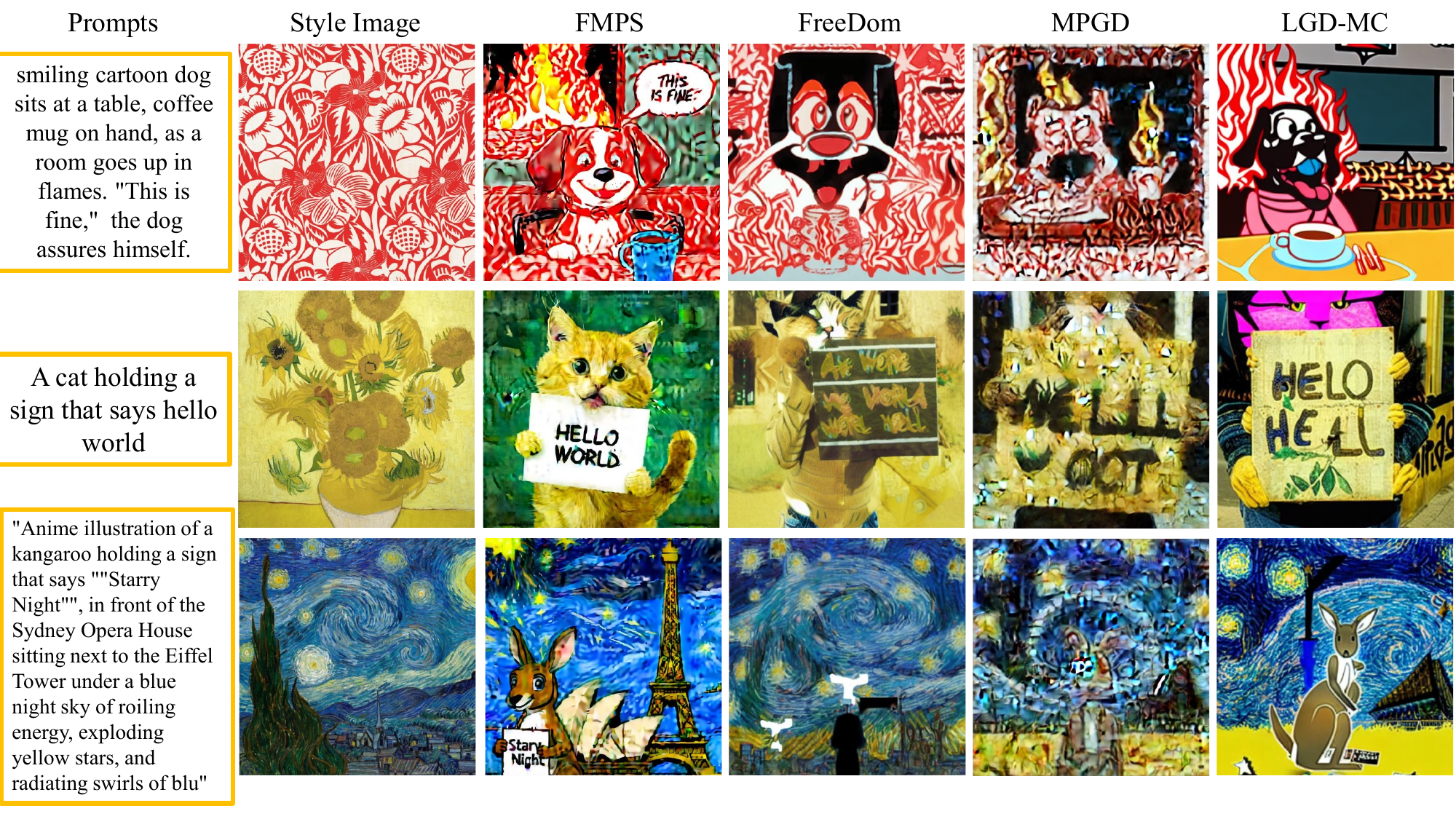}
    \caption{Quantitative comparison of complex style text-to-image. We compared all the previous methods on SDMs.}
    \label{fig:all_style}
\end{figure*}
We report the comprehensive quantitative results for both linear/nonlinear (shown in the supplementary) and style text-to-image (shown in Fig.~\ref{fig:all_style}). It could be found that FMPS outperforms the previous methods in all tasks.

\subsection{Ablation Study}
We explore the influence of the $r$. Then, we explore the performance of the FMPS-gradient FMPS-free.

\textbf{Trade-off between FMPS-gradient and FMPS-free}. To illustrate the trade-off between FMPS-gradient and FMPS-free, we report the ablation study in Table~\ref{fig:ala_trade}. The advantage of FMPS-free is the low time cost, while it achieves 2x speed-up after dropping the gradient. Meanwhile, its performance has decreased to a low level. In the end, we conducted an ablation study to show the validity of the normalization trick. It can be concluded that the normalization trick is essential for the performance of FMPS. After dropping it, the performance dramatically drops.
\begin{figure}
    \centering
    \includegraphics[width=0.8\linewidth]{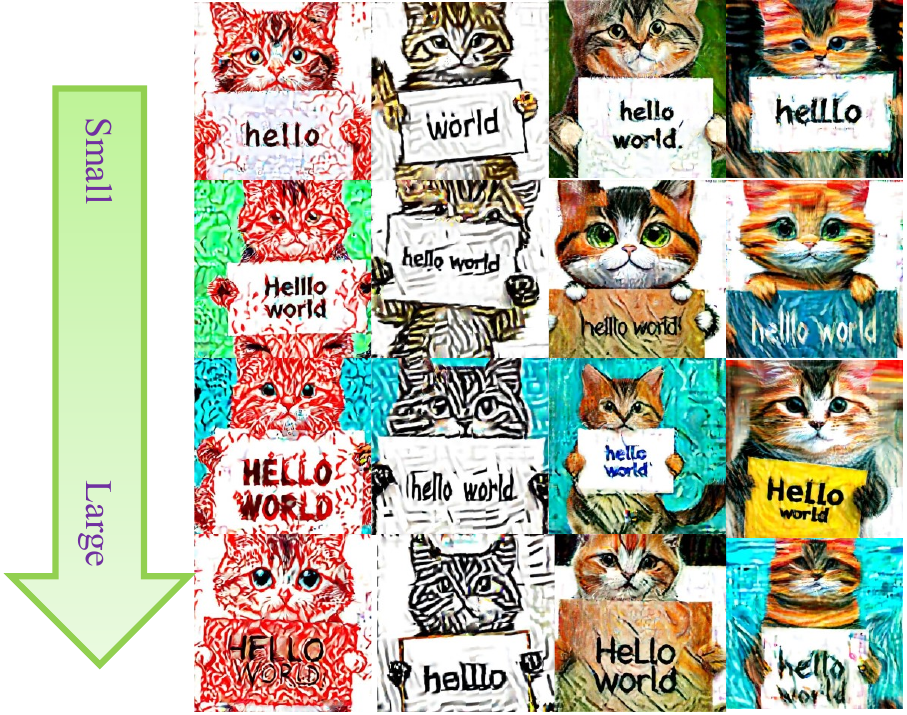}
    \caption{The qualitative results for different $r$. The prompt is ``A cat holding a sign that says hello world."}
    \label{fig:ala_rho}
\end{figure}

\textbf{Influence of $r$}. To illustrate the influence of $r$, we report the qualitative example based on the text-style generation tasks in Fig.~\ref{fig:ala_rho}. The $r$ is in $[0.3,2]$ from the small to the large. The experimental results show that FMPS works well in different $r$. The difference is the influence of the prompt. In the middle level, i.e., $r \in [0.5,1.5]$, FMPS could ensure that the generated images obey the prompt and change the style correctly. 
\section{Conclusion}
\label{sec:conclusion}

We proposed a novel training-free conditional method for FMs called FMPS. FMPS first introduces the condition into the velocity function by introducing a correction term, which offers the additional direction for the condition in the velocity field. Then, we formulate the correction term as the score function in SDMs, enabling the application of posterior sampling within the flow matching framework.

\textbf{Limitation.} There are some constraints for FMPS, similar to the posterior sampling-based works on DPMs. For example, the seed is the key to the generation, since if the seed is bad, the process will fail. Then, the color of the generated images will show a slight deviation compared to the DPM methods, which is an interesting finding since SSIM is not well reported in the linear inverse task. This may be due to the influence of stochastic and arbitrary paths. We will work on these further.

\bibliographystyle{unsrt}  
\bibliography{main}  

\end{document}